\documentclass[11pt]{article}
\usepackage{amsfonts,amsmath,amsthm,amssymb,enumerate,hyperref,MnSymbol,mathrsfs}
\DeclareMathOperator{\sgn}{sgn}

\def\E{\mathop{\mathrm{E}}}

\usepackage{fullpage}

\newcommand \poly {{\rm poly}}

\def\reals{{\mathbb R}}
\def\rn{\reals^n}

\newcommand{\ignore}[1]{}
\def\cP{\mathcal{P}}
\def\cU{\mathcal{U}}

\newtheorem{theorem}{Theorem}
\newtheorem{lemma}[theorem]{Lemma}

\newcommand{\myalg}[3]{
\begin{center}
\fbox{
\parbox{5.5in}{
Algorithm {\sc #1}.

#3
}}
\end{center}
}

\newcommand \cD {\mathcal{D}}
\newcommand \cT {\mathcal{T}}

\newcommand \cS {\mathscr{S}}

\def\01{\{0,1\}}
\def\11{\{-1,1\}}

\newcommand{\eps}{\epsilon}
\newcommand{\beq}{\begin{equation}}
\newcommand{\eeq}{\end{equation}}

\title{Decision trees are PAC-learnable from most product distributions: a smoothed analysis}

\author{Adam Tauman Kalai\\Microsoft Research New England \and Shang-Hua Teng\thanks{This work was done while the author was visiting Microsoft Research New England.}\\Boston University}
\begin{document}
\maketitle

\abstract{
We consider the problem of PAC-learning decision trees, i.e., learning a decision tree over the $n$-dimensional hypercube from independent random labeled examples.  Despite significant effort, no polynomial-time algorithm is known for learning polynomial-sized decision trees (even trees of any super-constant size), even when examples are assumed to be drawn from the {\em uniform} distribution on \{0,1\}$^n$.  We give an algorithm that learns arbitrary polynomial-sized decision trees for {\em most product distributions}.  In particular, consider a random product distribution where the bias of each bit is chosen independently and uniformly from, say, $[.49,.51]$.  Then with high probability over the parameters of the product distribution and the random examples drawn from it, the algorithm will learn any tree.  More generally, in the spirit of smoothed analysis, we consider an arbitrary product distribution whose parameters are specified only up to a $[-c,c]$ accuracy (perturbation), for an arbitrarily small positive constant $c$.
}

\section{Introduction}
Decision trees are classifiers at the center stage of both the theory and practice of machine learning.  Despite decades of research, no polynomial-time algorithm is known for PAC-learning polynomial-sized (or any super-constant-sized) Boolean decision trees over $\01^n$, even assuming examples are drawn from the uniform distribution on inputs.  The situation is no better for any other constant-bounded product distribution.  In light of this, what we show is perhaps surprising: {\em every} decision tree can be learned from {\em most} product distributions.  Hence, the uniform-distribution assumption common in learning (and other fields) may not be simplifying matters as one might hope.

\subsection{Related work}

Learning decision trees in Valiant's PAC model \cite{Valiant:84} requires learning an arbitrary tree from polynomially-many random labeled examples, drawn independently from an arbitrary distribution and labeled according to the tree.  Note that the output of the learning algorithm need not be a decision tree -- any function, which well approximates the target tree on future examples drawn from the same distribution as the training data, suffices.  The uniform-PAC model of learning assumes that data is drawn from the uniform distribution.  In previous work, size-$s$ trees were shown to be PAC-learnable in time $O\left(n^{\log s}\right)$ \cite{EH:89,Blum:92}.  Juntas, functions that depend on only $r$ ``relevant'' bits (a special case of decision trees of size $2^r$) can be uniform-PAC learned faster: in time roughly $O(n^{0.7r})$ \cite{MOS:03}.  A variety of alternatives to PAC learning have been considered, to circumvent the difficulties.  {\em Random} depth-$O(\log n)$ trees have been shown to be properly\footnote{The output of their algorithm is a decision-tree classifier.} learnable, with high probability, from uniform random examples by Jackson and Servedio \cite{JacksonServedio:03}.
  Decision trees have been also shown to be learnable from data which is coming from a {\em random walk}, i.e., consecutive training examples differ in a single random position \cite{BMO+:05}.  A seminal result of Kushilevitz and Mansour (KM) \cite{KushilevitzMansour:93}, using an algorithm similar to Goldreich-Levin \cite{GoldreichLevin:89}, shows that decision trees are uniform-PAC learnable
 from membership queries (i.e., black box access to the function) in polynomial time.  Since KM proved to be an essential ingredient in further work such as learning DNFs \cite{Jackson:97} and agnostic learning \cite{GKK:08}, as well as to applications beyond learning, the present work gives hope to a number of questions discussed in Section \ref{sec:conc}.

We consider a ``smoothed learning'' model inspired by Smoothed Analysis, which Spielman and Teng introduced to explain why the simplex method for linear programming (LP) usually runs in polynomial time \cite{SpielmanTeng:04}.  Roughly speaking, they show that if each parameter of an LP is perturbed by a small amount, then the simplex method will run in polynomial time with high probability (in fact, the expected run-time will be polynomial).  For LP's arising from nature or business (as opposed to reduction from another computational problem), the parameters are measurements or estimates that have some inherent inaccuracy or uncertainty. Hence, the model is reasonable for a large class of interesting LP's.

\subsection{Main result}

We suppose that the examples are coming from a product distribution $\cP_\mu$, specified by $\mu \in [0,1]^n$ where $\mu_i = \E_{x \sim \cP_\mu}[x_i]$.  An illustrative instantiation of our main result is the following.  Take any decision tree and pick a random $\mu \in [0.49,0.51]^n$.  Then, with high probability (over $\mu$ and the random examples from $\cP_\mu$), our algorithm will output a polynomial threshold function which is a good approximation to the tree.  Since $\cP_{(.5,\ldots,.5)}$ is the uniform distribution, the choice of $\mu \in [0.49,0.51]^n$ is close {\em in spirit}\footnote{Statistically speaking, this distribution is quite different than the uniform distribution.  Learning  form {\em any} $\mu \in \left[1/2-\sqrt{1/n},1/2+\sqrt{1/n}\right]^n$ would likely be as difficult as learning from the uniform distribution.} to the uniform distribution.

More generally, fix any arbitrarily small constant $c \in (0,1/4)$.  An adversary, if you will, chooses an arbitrary decision tree $f$ and an arbitrary $\bar\mu \in [2c,1-2c]^n$ but the actual product distribution will have parameters $\mu = \bar\mu+\Delta$, where $\Delta \in [-c,c]^n$ is a uniformly random perturbation.  Then, a polynomial number of examples will be drawn from $\cP_\mu$.  With high probability over the perturbation $\Delta$ and the data drawn from $\cP_{\bar\mu+\Delta}$, the algorithm will output a function which is very close to $f$.  The main theorem we prove is the following.

\begin{theorem}\label{th:MAIN}
Let $c \in (0,1/4)$.  Then there is a univariate polynomial $q$ such that, for any integers $n,s \geq 1$, reals $\eps,\delta>0$, function $f: \01^n \rightarrow \11$ computed by a size-$s$ decision tree, and any $\bar\mu \in [2c,1-2c]^n$, with probability $\geq 1-\delta$ over $\Delta$  chosen uniformly at random from $[-c,c]^n$ and $m \geq q(ns/(\delta\eps))$ training examples $(x_1,f(x_1)),\ldots,(x_m,f(x_m))$ where each $x_i$ is drawn independently from $\cP_\mu$ (where $\mu =\bar\mu + \Delta$), the output of algorithm $L$ is $h$ with, $$\Pr_{x \sim \cP_\mu}[h(x)\neq f(x)] \leq \eps.$$
Algorithm $L$ is polynomial time, i.e., it runs in time $\poly(n,m)$ and outputs a polynomial threshold function.
\end{theorem}

It is worth making a few remarks about this theorem.  Worst-case analysis is beautiful but sometimes leads to artificial limitations, especially in domains like learning where we do not actually believe that an {\em adversary} chooses the problem.  In this sense, it is natural to slightly weaken the power of the adversary.  Here, we have assumed that the adversary can only specify the product distribution up to $[-c,c]$ accuracy or rather that the adversary may have a {\em trembling hand} (to misuse a term of Selten \cite{Selten:75}).  As an example of smoothed analysis, ours is interesting because unlike linear programming, where worst-case polynomial-time alternatives to the simplex were already known, there are no known efficient algorithms for uniform-PAC learning decision trees.

In learning, the standard uniform-PAC model already ``assumes away'' any adversarial connection between the function being learned and the distribution over data.  Now, the uniform distribution assumption is made with the hope that the resulting algorithms may be useful for learning or at least shed light on issues involved in the problem; it is a natural first step in designing general-distribution learning assumptions.  We hope that the smoothed analysis serves a similar purpose.

\subsection{The approach}

The intuition behind our algorithm is quite simple.  It will turn out to be notationally convenient to consider examples $x \in \11^n$.  Now for starters, consider a decision tree that computes a $\log(n)$-sized parity $f(x)=\prod_{i \in S} x_i$, for some set $S \subseteq \{1,2,\ldots,n\}, |S|=\log_2(n)$.  This can be done using a size $n$ tree.  Under the uniform distribution on examples, each bit $x_i$ (or any subset of $\leq \log(n)-1$ bits) is uncorrelated with $f$.  Now take a product distribution with random mean vector $\mu \in [-c,c]^n$ and define $x' = x-\mu$, so that $\E[x_i]=0$.  Then with probability $\geq 1-\delta$, $f(x)$ has a significant ($poly(\delta/n)$) correlation with each $x_i'$ for $i \in S$ and no correlation with any $i \not\in S$.  Hence, it is easy to find the relevant bits.  Now, a polynomial size-tree may, in general, involve all $n$ bits so finding the relevant bits is not sufficient.

As is standard for Fourier learning under product distributions, one can write $f(x)=f(x')$ as a polynomial in $x'$.  Each coefficient of a term $\prod_{i \in S}x'_i$ can be estimated in a straightforward manner from random examples.  However, finding the {\em heavy} coefficients (those with large magnitude) is a bit like finding a number of needles in a haystack.  However, this is the most fascinating aspect of the problem -- it requires so-called {\em feature discovery} or {\em feature construction} algorithms.  These algorithms hence tie together a fundamental problem in both the theory and practice of learning: many claim that the heart of the problem of machine learning is really that of finding or creating good features \cite{Mitchell:97}.

The key property we prove is the following, with high probability over $\mu \in [-c,c]^n$.  If the coefficient in $f(x')$ of a term $\prod_{i \in T} x'_i$ is large, then so is the coefficient of $\prod_{i \in S}x'_i$ for each $S \subseteq T$.  This makes finding all the large coefficients easy using a top-down approach. The proof of this fact relies on two properties: there is a simple relationship between different coefficients under different product distributions, and a low-degree nonzero multilinear polynomial cannot be too close to 0 too often (this is a continuous generalization of the Schwartz-Zippel theorem).  In our simple example, it is easy to see that by expanding $f(x)=\prod_{i \in S}x_i=\prod_{i \in S}(x_i'+\mu_i)$, {\em all} coefficients of terms $\prod_{i \in T}x_i'$, for $T \subseteq S$, will be nonzero with probability 1.

Another perspective on the algorithm is that it gives a substitute for KM (equivalently Goldreich-Levin) using {\em random examples} instead of adaptive queries.  It is a weaker substitute in that it is only capable of finding large coefficients on terms of $O(\log n)$.

\section{Organization}
Preliminaries are given in Section \ref{sec:prelim}.  Before we give the smoothed algorithm for learning, we prove a property about Fourier coefficients under random product distributions in Section \ref{sec:key}.  We then give the algorithm and analysis in Section \ref{sec:alg}.  Conclusions and future work are discussed in Section \ref{sec:conc}.

\section{Preliminaries}\label{sec:prelim}
Let $N=\{1,2,\ldots,n\}$.  As mentioned, for notational ease we consider examples $(x,y)$ with  $x\in \11^n$ and $y\in \11$.
For $S \subseteq N$, $x \in \rn$, let $x_S$ denote $\prod_{i \in S}x_i$.  Any function $f:\11^n \rightarrow \reals$ can be written uniquely as a multilinear polynomial in $x$,
\begin{equation*}
f(x)=\sum_{S \subseteq N} \hat{f}(S) x_S.
\end{equation*}
The $\hat{f}(S)$'s are called the Fourier coefficients.  The degree of a multilinear polynomial is $\deg(f)=\max \{|S|~|~\hat{f}(S)\neq 0\}$, and with a slight abuse of terminology, we say a polynomial is degree-$d$ if $\deg(f)\leq d$.

Henceforth we write $\sum_S$ to denote $\sum_{S \subseteq N}$ and $\sum_{|S|=d}$ to denote the sum over $S \subseteq N$ such that $|S|=d$. Similarly for $\sum_{|S|>d}$, and so forth.  We write $x \in_\cU A$ to denote $x$ chosen uniformly at random from set $A$.  One may define an inner product between functions $f,g:\11^n \rightarrow \reals$ by, $\langle f,g \rangle=\E_{x \in_\cU \11^n}[f(x)g(x)]$.  It is easy to see that $\langle x_S,x_T\rangle$ is 1 if $S=T$ and 0 otherwise.  Hence, the $2^n$ differen $x_S$'s form an orthonormal basis for the set of real-valued functions on $\11^n$.  We thus have that $\langle f, g \rangle = \sum_{S\subseteq N} \hat{f}(S)\hat{g}(S)$, and Parseval's equality,
$$\langle f,f \rangle = \sum_{S \subseteq N} \hat{f}^2(S)=\E_{x \in_\cU \11^n}[f^2(x)].$$
This implies that for any $f:\11^n \rightarrow [-1,1]$, $\sum_S \hat{f}^2(S)\leq 1$.  It is also useful for bounding $\E[(f(x)-g(x))^2] = \sum_S (\hat{f}(S)-\hat{g}(S))^2$.

A product distribution $\cD_\mu$ over $\11^n$ is parameterized by its mean vector $\mu \in [-1,1]^n$, where $\mu_i=\E_{x \sim \cD_\mu}[x_i]$ and the bits are independent. (We now use $\cD$ to avoid confusion with product distributions $\cP$ over $\01^n$ discussed in the introduction.) The uniform distribution is $\cD_0$.  We say $\cD_\mu$ is $c$-bounded if $\mu_i \in [-1+c,1-c]$ for all $i$.
Fix any constant $c \in (0,1/2)$.  We assume we have some fixed $2c$-bounded product distribution $\bar\mu \in [-1+2c,1-2c]^n$ and that a random {\em perturbation} $\Delta \in [-c,c]^n$ is chosen uniformly at random and the resulting product distribution has $\mu=\bar\mu+\Delta$.  Note that $\cD_\mu$ is $c$-bounded and called the {\em perturbed} product distribution.

For any distribution $\cD$ on $\11^n$, one can similarly define an inner product $\langle f,g \rangle_{\cD}=\E_{x \sim \cD}[f(x)g(x)]$. In the case of a product distribution $\cD_\mu$, it is natural to normalize the coordinates so that they have mean 0 and variance 1.  Let $z(x,\mu) \in \rn$ be the vector defined by $z_i(\mu,x) = (x_i-\mu_i)/{\sqrt{1-\mu_i^2}}$.  When $\mu$ and $x$ are understood from context, we write just $z$. This normalization gives $\E_{x \sim \cD_\mu}[z_i(x,\mu)]=0$ and $\E_{x \sim \cD_\mu}[z_i^2(x,\mu)]=0$.  Let $z_S=z_S(x,\mu)=\prod_{i\in S}z_i(x,\mu)$.  It is also easy to see that  $\E_{x \sim \cD_\mu}[z_S z_T]$ is 1 if $S=T$ and 0 otherwise.  Hence, the $2^n$ differen $x_S$'s form an orthonormal basis for the set of real-valued functions on $\11^n$ with respect to $\langle\rangle_{\cD_\mu}$.  We define the normalized Fourier coefficient, for any $S \subseteq N$,
\begin{equation}\label{eq:qe}
\hat{f}(S,\mu)=\E_{x \sim \cD_\mu}[f(x)z_S(x,\mu)].
\end{equation}
Note that this gives a straightforward means of estimating any such coefficient.
Also observe that $\hat{f}(S,0)=\hat{f}(S)$ and that, for any $\mu \in [-1,1]^n$,
$$f(x)=\sum_S \hat{f}(S,\mu) z_S(x,\mu).$$
Finally, it will be convenient to define a partially normalized Fourier coefficient,
$$\bar{f}(S,\mu)=\frac{\hat{f}(S,\mu)}{\prod_{i \in S} \sqrt{1-\mu_i^2}}.$$
Note that if $\mu \in [-1+c,1-c]^n$ then we have,
\begin{equation}
\label{eq:last}
|\hat{f}(S,\mu)|
\leq
|\bar{f}(S,\mu)|
\leq
\frac{|\hat{f}(S,\mu)|}{(1-(1-c)^2)^{|S|/2}} \leq \frac{|\hat{f}(S,\mu)|}{c^{|S|/2}}
\end{equation}
In this notation, we also have,
$$f(x)=\sum_S \bar{f}(S,\mu)\prod_{i \in S} (x_i-\mu_i) = \sum_S \bar{f}(S,\mu)(x_i-\mu_i)_S
$$
Hence, for any $\mu = \bar\mu+\Delta$,
$$
\sum_S \bar{f}(S,\mu) (x -\mu)_S=\sum_S \bar{f}(S,\bar\mu)\bigl((x -\mu)+\Delta\bigr)_S.$$
Collecting terms gives a means for translating between product distributions $\mu=\bar\mu+\Delta$:
\begin{equation}\label{eq:bar}
\bar{f}(S,\mu) = \sum_{T \supseteq S} \bar{f}(T,\bar{\mu}) \Delta_{T \setminus S}
\end{equation}

\subsection{Decision trees}

A decision tree $\cT$ over $\11^n$ is a rooted binary tree, in which each internal node is labeled with an integer $i \in N$, and each leaf is assigned a label of $\pm 1$.  We consider Boolean decision trees, in which case each internal node has exactly two children, and the two outgoing edges are labeled, one of them $1$ and the other $-1$.  The tree computes a function $f_\cT:\11^n\rightarrow \11$ defined recursively as follows.  If the root is a leaf, then the value is simply the value of the leaf.  Otherwise, say the root is labeled with $i$, and say it's children are $\cT_{-1}$ and $\cT_1$, following the labels $-1$ and $+1$, respectively.  The the value of the tree is defined to be the value computed by $\cT_{x_i}$ on $x$, i.e., $f_{\cT_{x_i}}(x)$.  In other words,
$$f(x)=\left(\frac{1}{2}+\frac{x_i}{2}\right)f_{\cT_{1}}(x)+\left(\frac{1}{2}-\frac{x_i}{2}\right)f_{\cT_{-1}}(x).$$
We assume that no node appears more than once on any path down from the root to a leaf.  Hence, the above function is a multilinear polynomial $f: \11^n\rightarrow \11$, but more in some cases it may be helpful to think of it as simply a multilinear polynomial $f:\rn\rightarrow \reals$.  The size of a decision tree is defined to be the number of leaves.  We define the depth of the root of the tree to be 0.  Thus a depth-$d$ tree computes a degree-$d$ multilinear polynomial.

\section{Fourier properties for random product distributions}\label{sec:key}

The following lemmas show that, with high probability, for every coefficient $\hat{f}(S)$ that is sufficiently large, say $|\hat{f}(S)| > b$, it is very likely that all subterms $T \subseteq S$ have $|\hat{f}(T)|>a$, for some $a < b$.  It turns out that this is easier to state in terms of the partially normalized coefficients $\bar{f}(S)$.  The following simple lemma is at the heart of the analysis.
\begin{lemma}\label{lem:1}
Take any $c \in (0,1/2)$, $\bar\mu \in [-1+c,1-c]^n$ and let $\mu = \bar\mu + \Delta$, where $\Delta$ is chosen uniformly at random from $[-c,c]^n$.  Let $f:\reals^n\rightarrow \reals$ be any multilinear function $f(x)=\sum_{S} \bar{f}(S,\mu) (x-\mu)$.  Then for any
$T \subseteq U\subseteq N$, $a,b>0$,
$$\Pr_{\Delta \in_\cU [-c,c]^n}[|\bar{f}(T,\mu)|\leq a ~\bigl|~|\bar{f}(U,\mu)|\geq b] \leq \sqrt{\frac{a}{b}}(4/c)^{|U \setminus T|/2}.$$
\end{lemma}
(For events $A,B$, we define $\Pr[A|B]=0$ in the case that $\Pr[B]=0$.)
In order to prove lemma \ref{lem:1}, we give a continuous variant of Schwartz-Zippel theorem.  This lemma states that a nonzero degree-$d$ multilinear function cannot be too close to 0 too often over $x \in [-1,1]^n$.
\begin{lemma}\label{SZ}
Let $g:\rn\rightarrow \reals$ be a degree-$d$ multilinear polynomial, $g(x) = \sum_{|S|\leq d} \hat{g}(S)x_S$.  Suppose that
 there exists $S \subseteq N$ with $|S|=d$ and $|\hat{g}(S)| \geq 1$.  Then for a uniformly chosen random $x \in [-1,1]^n$, and for any $\eps>0$, we have,
$$\Pr_{x \sim_\cU [-1,1]^n}\left[~\left|g(x)\right| \leq \eps~\right] \leq 2^d \sqrt{\eps}.$$
\end{lemma}
\begin{proof}
WLOG let say $\hat{g}(D)= 1$ for $D=\{x_1,x_2,\ldots,x_d\}$ for we can always permute the terms and rescale the polynomial so that this coefficient is exactly 1.  We first establish that,
\begin{equation}\label{eq:mon}
\Pr_{x \in_\cU [-1,1]^n}[|g(x)|\leq\eps] \leq \Pr_{x \in_\cU [-1,1]^n}[\left|x_D\right|\leq \eps].
\end{equation}
In other words, the worst case is a monomial.  To see this, write,
$$g(x)=x_1 g_1(x_2,x_3,\ldots,x_n) + g_2(x_2,x_3,\ldots,x_n).$$
Now, by independence imagine picking $x$ by first picking $x_2,x_3,\ldots,x_n$ (later we will pick $x_1$).  Let $\gamma_i=g_i(x_2,\ldots,x_n)$ for $i=1,2$.  Then, consider the two sets $I_1= \{x_1 \in \reals : |x_1 \gamma_1+\gamma_2| \leq \eps\}$ and $I_2=\{x_1 \in \reals:|x_1\gamma_1|\leq \eps\}$.  These are both intervals, and they are of equal width.  However, $I_2$ is centered at the origin.  Hence, since $x_1$ is chosen uniformly from $[-1,1]$, we have that for any fixed $\gamma_1,\gamma_2$, $\Pr_{x_1 \in_\cU [-1,1]}[x_1 \in I_1] \leq \Pr_{x_1 \in_\cU [-1,1]}[x_1 \in I_2]$, because $I_2 \cap [-1,1]$ is at least as wide as $I_1 \cap [-1,1]$.  Hence it suffices to prove the lemma for those functions where $\hat{g}(S)=0$ for all $S$ for which $1 \notin S$.  (In fact, this is the worst case.) By symmetry, it suffices to prove the lemma for those functions where $\hat{g}(S)=0$ for all $S$ for which $i \notin S$, for $i=1,2,\ldots,d$. After removing all terms $S$ that do not contain $D$ we are left with the function $x_D$, establishing (\ref{eq:mon}).  Now, for a loose bound, one can use Markov's inequality:
$$\Pr[|x_D|\leq \eps] = \Pr\left[|x_D|^{-1/2} \geq \eps^{-1/2}\right] \leq \frac{\E[|x_D|^{-1/2}]}{\eps^{-1/2}}=\eps^{1/2}2^d.$$
In the last step, $\E[|x_D|^{-1/2}]=\E[|x_1|^{-1/2}]^d$ by independence and symmetry, and a simple calculation based on the fact that $|x_1|$ is uniform from $[0,1]$ gives $\E[|x_1|^{-1/2}]=2.$  Although we won't use it, we mention that one can compute a tight bound, $\Pr[|x_1\ldots x_d|\leq \eps]=\eps \sum_{i=0}^{d-1} \log^{i}\frac{1}{\eps}.$  This is shown by induction and
$\Pr[|x_1x_2\ldots x_{i+1}|\leq \eps] = \int_0^1 \Pr[|x_1x_2 \ldots x_i\cdots | \leq \frac{\eps}{t}]dt.$
\end{proof}

With this lemma in hand, we are now ready to prove Lemma \ref{lem:1}.

\begin{proof}[Proof of Lemma \ref{lem:1}]
For any set $S\subseteq N$, let $\Delta=(\Delta[S],\Delta[N\setminus S])$ where $\Delta[S] \in [-c,c]^{|S|}$ represents the coordinates of $\Delta$ that are in $S$.  Let $V = U \setminus T$. The main idea is to imagine picking $\Delta$ by picking $\Delta[N\setminus V]$ first (and later picking $\Delta[V]$).  Now, we claim that once $\Delta[N\setminus V]$ is fixed, $\bar{f}(U,\mu)$ is determined.  This follows from (\ref{eq:bar}), using the fact that $S \setminus U \subseteq N \setminus V$:
$$\bar{f}(U,\mu)=\sum_{S \supseteq U}\bar{f}(S,0)\mu_{S \setminus U}.$$
On the other hand $\bar{f}(T,\mu)$ is not determined only from $\Delta[N\setminus V]$.  Once we have fixed $\Delta[N\setminus V]$, it is now a polynomial in $\Delta[V]$ using (\ref{eq:bar}) again:
$$g(\Delta[V])=\bar{f}(T,\mu) = \sum_{S \supseteq T} \bar{f}(S,\bar\mu)\Delta_{S \setminus T}.$$
Clearly $g$ is a multilinear polynomial of degree at most $|V|$.  Most importantly, the coefficient of $\Delta_V$ in $g$ is exactly $\sum_{S \supseteq T \cup V}\bar{f}(S,\bar\mu)\Delta_{S \setminus (T \cup V)}=\bar{f}(U,\mu)$, since $T \cup V=U$.
Hence, the choice $\bar{f}(S,\mu)$ can be viewed as a degree-$d$ polynomial in the random variable $\Delta[V]$ with leading coefficient $\bar{f}(U,\mu)$, and we can apply Lemma \ref{SZ}.  So, suppose that $|\bar{f}(U,\mu)|>b$.  Let $g'(x) = b^{-1}c^{-|V|}g(x c)$, so the coefficient of $x_V$ in $g'$ is $(b^{-1}c^{-|V|})c^{|V|} \bar{f}(U,\mu) \geq 1$.  By lemma \ref{SZ},
$$\Pr_{\Delta[V] \in_\cU [-c,c]^{|V|}}[|g(\Delta[V])|\leq a]=\Pr_{x \in_\cU [-1,1]^{|V|}}[|g'(x)|<ab^{-1}c^{-|V|}] \leq \sqrt{\frac{a}{b}}c^{-|V|/2}2^{|V|}.\qedhere$$
\end{proof}

We now observe that Lemma \ref{lem:1} implies that with high probability, all sub-coefficients of large $\hat{f}(S)$ will be pretty large.
\begin{lemma}\label{lem:2}
Let $f:\11^n\rightarrow [-1,1]$.  Let $\alpha,\beta \geq 0$, $d \in \mathbb{N}$.  Let $c\in (0,1/2)$, $\bar\mu \in [-1+2c,1-2c]^n$, and $\mu = \bar\mu + \Delta$ where $\Delta \in [-c,c]^n$ is chosen uniformly at random.  Then,
$$\Pr_{\Delta \in_{\cU} [-c,c]^n}\left[\exists T \subseteq U \subseteq N \text{ such that }|U|\leq d \wedge |\hat{f}(T,\mu)|\leq \alpha \wedge |\hat{f}(U,\mu)|\geq \beta\right]\leq \alpha^{1/2}\beta^{-5/2}(2/c)^{2d}.$$
\end{lemma}
\begin{proof}
Since $\mu$ is $c$-bounded, for any $S \subseteq N$ with $|S|\leq d$, $|\hat{f}(S,\mu)| \leq |\bar{f}(S,\mu)| \leq c^{-d/2}|\hat{f}(S,\mu)|$, (see (\ref{eq:last})), it suffices to show that, for any $a, b>0$,
$$\Pr_{\Delta \in_{\cU} [-c,c]^n}\left[\exists T \subseteq U \subseteq N \text{ such that }|U|\leq d \wedge |\bar{f}(T,\mu)|\leq a \wedge |\bar{f}(U,\mu)|\geq b\right]\leq a^{1/2}b^{-5/2}4^dc^{-3d/2}.$$
This is because for $a=\alpha c^{-d/2}$ and $b=\beta$, $|\hat{f}(U,\mu)|\geq \beta$ implies $|\bar{f}(U,\mu)|\geq b$,  and
$|\hat{f}(T,\mu)|\leq \alpha$ implies $|\bar{f}(U,\mu)|\leq a$.
We can bound the above quantity by the union bound using Lemma \ref{lem:1}.  It is at most,
\begin{align*}
\sum_{\begin{array}{c} _{|U|\leq d}\\^{T \subseteq U}\end{array}} \Pr[|\bar{f}(T,\mu)|\leq a \wedge |\bar{f}(U,\mu)|\geq b] &=
\sum_{\begin{array}{c} _{|U|\leq d}\\^{T \subseteq U}\end{array}} \Pr[|\bar{f}(T,\mu)|\leq a ~\bigl|~|\bar{f}(U,\mu)|\geq b] \Pr[ |\bar{f}(U,\mu)|\geq b]\\
&\leq \sum_{|U|\leq d} \sum_{T \subseteq U} a^{1/2}b^{-1/2}(4/c)^{|U\setminus T|/2}\Pr[ |\bar{f}(U,\mu)|\geq b]\\
&\leq 2^d a^{1/2}b^{-1/2}(4/c)^{d/2} \sum_{|U|\leq d} \Pr[ |\bar{f}(U,\mu)|\geq b]\\
&= 2^da^{1/2}b^{-1/2}(4/c)^{d/2} \E\bigl[\left|\{U~|~|U|\leq d \wedge |\bar{f}(U,\mu)|\geq b\}\right|\bigr]
\end{align*}
All probabilities in the above are over $\Delta \in_\cU [-c,c]^n$.  Finally, there can be at most $c^{-d}b^{-2}$ different $U \subseteq N$ such that $|\bar{f}(U,\mu)|\geq b$ since $\sum_S \bar{f}^2(S,\mu)\leq c^{-d} \sum_S \hat{f}^2(S,\mu) \leq c^{-d}$ for all $\mu$ by Parseval's inequality.  Hence, the expected number of such $U$ is at most $c^{-d}b^{-2}$ and we have the lemma.
\end{proof}

\section{Algorithm}\label{sec:alg}

For simplicity, we suppose that the algorithm has exact knowledge of $\mu$.  In general, these parameters can be estimated to any desired inverse-polynomial accuracy in polynomial time.  The algorithm is below.

\myalg{L}{L}{

Inputs: $(x^1,y^1),\ldots,(x^m,y^m) \in \reals^n \times \11$ and $\mu \in [c,1-c]^n$.

\begin{enumerate}

\item Let $z^j_i:=\frac{x^j_i-\mu_i}{\sqrt{1-\mu_i^2}}$, for $i=1,2,\ldots,n$ and $j=1,2,\ldots,m$.

\item Let $\cS_0:=\{\emptyset\}$.

\item For $d=1,2,\ldots,\frac{\log m}{12}(1-\max_{i\leq n}{|\mu_i|}):$

\begin{enumerate}

\item Let
  $$\cS_d:=\cS_{d-1} \cup \left\{ S \cup \{i\} ~\bigl|~ S \in \cS_{d-1} \wedge \left|\frac{1}{m}\sum_{j=1}^m y^j z^j_{S \cup \{i\}} \right| \geq m^{-1/3} \right\}.$$

\item If $|\cS_d|>m$ then abort and output FAIL.

\end{enumerate}

\item Let $p$ be the following polynomial $p: \11^n \rightarrow \reals$,
$$p(x)=\sum_{S \subseteq \cS_n} \left(\frac{1}{m}\sum_{j=1}^m y^j z^j_S\right)\chi_S(z).$$

\item Output $h(x)=\sgn(p(x))$.

\end{enumerate}
}

It is well-known that functions computed by decision trees can be approximated by sparse polynomials, namely, the set of ``heavy''
 coefficients, i.e., those which have large magnitudes. These heavy coefficients tend to be on terms of small degree as well.  This is true for any constant bounded product distribution.
\begin{lemma}\label{lem:throw}
Let $c\in [0,1/2]$, let $\mu \in [-1+c,1-c]^n$, $d \in \mathbb{N}$, $\beta>0$, and let $f:\11^n\rightarrow\11$ be computed by a size-$s$ decision tree.  Then,
$$\sum_{S: |\hat{f}(S,\mu)|\geq \beta \wedge |S|\leq d} \hat{f}^2(S)\geq 1-\left(4(1-c/2)^ds+2^{d+2}\beta\right).$$
\end{lemma}
Hence, it is to be shown that algorithm $L$ identifies these heavy coefficients and estimates them well.   The proof of this lemma is deferred until after the proof of the main theorem.

\begin{proof}[Proof of Theorem \ref{th:MAIN}]
First, note that for any $g:\11^n \rightarrow \reals$ and any distribution $\cD$ over $\11^n$, $\Pr_{x \sim \cD}[\sgn(g(x))\neq f(x)] \leq \E_{x \sim \cD}[(g(x)-f(x))^2]$.  The reason is that any time $\sgn(g(x)) \neq f(x)$, we have that $|g(x)-f(x)|\geq 1$, since $f:\11^n \rightarrow \11$.  Hence, it suffices to show that with probability $\geq 1-\delta$, $$\E_{x \sim \cD_\mu}[(p(x)-f(x))^2] = \sum_S (\hat{p}(S,\mu)-\hat{f}(S,\mu))^2\leq \eps.$$
This is what we do. Define the estimate of $\hat{f}(S,\mu)$ (based on the data) to be,
$$e(S)=\frac{1}{m}\sum_{j=1}^m y^j z^j_S.$$
By equation (\ref{eq:qe}), we have that $\E[e(S)]=\hat{f}(S,\mu)$, for any fixed $S, \mu$, where the expectation is taken over the $m$ data points.  Of course, steps (3a) and (4) only evaluate $e(S)$ on a small number of sets, but it is helpful to define $e$ for all $S$.

Let $d=\frac{2}{c}\log\frac{12s}{\eps}$, $D=\frac{\log m}{12}(1-\max_{i\leq n}{|\mu_i|})$, $\beta = (\eps/(12s))^{1+2/c}$, $t=m^{-1/3}$, and $\tau =\frac{t\sqrt{\eps}}{4}$.  Note that $D \geq \frac{\log m}{12}c >d$ for $m=\poly(s/\eps)$, so the algorithm will at least attempt to estimate all coefficients up to degree $d$.

We define the set of {\em gingerbread features} to be,
$$G = \left\{S \subseteq N ~\bigl|~|S|\leq d \wedge |\hat{f}(S,\mu)|\geq \beta\right\}.$$
These are the features that we really require for a good approximation.  We define the set of {\em breadcrumb features} to be,
$$B = \left\{B \subseteq S ~\bigl|~S \in G\right\}.$$
These are the features which will help us find the gingerbread features.  The set of {\em pebble features} is,
$$P= \{\emptyset\} \cup \left\{S \subseteq N ~\bigl|~|S| \leq D, ~|\hat{f}(S,\mu)|\geq t-\tau\right\}.$$
These are the features that might possibly be included in $\cS_n$ on a ``good'' run of the algorithm.  Note that, by Parseval's inequality, $|P|\leq 1+(t-\tau)^{-2} \leq 1+2t^{-2} \leq 3t^{-2}$.
We will argue that, with high probability, $G \subseteq \cS_n \subseteq P$.  In order to do this, we also consider the set of {\em candidate features},
$$C = P \cup \left\{S \cup \{i\} ~\bigl|~S \in P,~ i \in N\right\}.$$
These are the set of all features that we might possibly estimate (evaluate $e(S)$) on a ``good'' run of the algorithm.
Let us formally call a run of the algorithm ``good'' if,
(a) $|\hat{f}(S,\mu)-e(S)| \leq \tau$ for all $S \in C$ and (b) $|\hat{f}(S,\mu)|\geq t+\tau$ for all $S \in B$.
First, we claim that (a) implies $\cS_n \subseteq P$.  This can be seen by induction, arguing that $\cS_i \subseteq P$ for all $i=0,1,\ldots,n$.  This is trivial for $i=0$.  If it holds for $i$, then for $i+1$, we have that the set of features on iteration $i$ that are estimated will all be in $C$, hence will all be within $\tau$ of correct.  Hence, for any of these features that is not in $P$, we will have $|e(s)| <t$ and it will not be included in $\cS_i$.
Second we claim that (a) and (b) imply that $B \subseteq \cS_n$.  The proof of this is similarly straightforward by induction.    So (a) and (b) imply that $G \subseteq \cS_n \subseteq P$, since $G \subseteq B$.  Note that since $|P| \leq 3t^{-2} < m$, the algorithm will not abort and output FAIL in this case.
Now,
$$\sum_S (\hat{p}(S,\mu)-\hat{f}(S,\mu))^2 \leq \sum_{S \in \cS_n} (e(S)-\hat{f}(S,\mu))^2 + \sum_{S \not\in B} \hat{f}^2(S,\mu) \leq |P|\tau^2 + 4(1-c/2)^ds+2^{d+2}\beta.$$
This follows from $|\cS_n| \leq |P|$ and Lemma \ref{lem:throw}.  Hence, a good run has,
$$\sum_S (\hat{p}(S,\mu)-\hat{f}(S,\mu))^2 \leq 3t^{-2}\tau^2 + 4(1-c/2)^ds+2^{d+2}\beta \leq \eps,$$
for the choice of parameters above, because $3t^{-2}\tau^2=(3/16)\eps$, $4(1-c/2)^ds \leq \eps/3$, and $2^{d+2}\beta \leq \eps/3$.  This means that every good run outputs a hypothesis of error $\leq \eps$.  It remains to show that the probability of a good run is at least $1-\delta$, which we do by the union bound over the two events (a) and (b).  By Lemma \ref{lem:2} property (b) fails with probability at most,
$$(t+\tau)^{1/2}\beta^{-5/2}(2/c)^{2d} \leq 2m^{-1/6} (12s/\eps)^{c'} \leq \delta/2,$$ for some constant $c'$ and $m=\poly(ns/(\delta\eps))$.  Finally, it remains to show that (a) fails with probability at most $\delta/2$.  First, we need to bound $|z^j_S|$ for each $S \in C$.  Let $v = 1-\max_{i\leq d} |\mu_i| \in [c,1]$ so that $D=\frac{\log m}{12}v$
We first observe that $|z_i(x,\mu)| \leq \frac{2-v}{\sqrt{1-(1-v)^2}} \leq 2/v$ for any $i \in N$, and $x \in \11^n$, by the definition of $z$.  This means that $|z_S(x,\mu)|\leq (2/v)^{\frac{\log m}{12}v}\leq m^{1/12}$ for all $S \in C$, $x \in \11^n$, using the fact that $(2/v)^{v}\leq e$ for all $v \leq 1$.  Finally, by Chernoff-Hoeffding bounds, the probability of $|e(S)-\hat{f}(S,\mu)|\geq \tau$ on any $S \in C$ is at most $2e^{-m\tau^2/(2m^{1/6})}.$  Since $|C| \leq n |P| \leq 3nt^{-2}$, it suffices to show that this is at most $\delta/(2|C|) \geq \delta t^2/(6n)$.  In other words, to finish, we need that $2e^{-m^{1/6}\eps/32} \geq \delta m^{-2/3}/(6n)$, which is clearly true for $m$ sufficiently large, in particular $\poly(ns/(\delta/\eps))$ certainly suffices.
\end{proof}

We now prove Lemma \ref{lem:throw}.
\begin{proof}[Proof of Lemma \ref{lem:throw}]
Let $g:\11^n\rightarrow \{-1,0,1\}$ be the function computed by the truncated decision tree in which each internal node at depth $d$ has been replaced by a leaf of value $0$.  Then,
$$\sum_S (\hat{f}(S,\mu)-\hat{g}(S,\mu))^2 = \E_{x \sim \cD_\mu}[(f(x)-g(x))^2] = \Pr_{x \sim \cD_\mu}[f(x)\neq g(x)] \leq (1-c)^d s.$$
The last inequality follows from the fact that the probability of reaching any leaf at depth $d$ is at most $(1-c)^d$.
Since $g$ is degree $d$, $\sum_{|S|>d} \hat{f}^2(S,\mu) \leq (1-c)^ds$.  Thus by removing all terms of degree greater than $d$, we throw out at most $(1-c)^ds$ mass.  Hence, it suffices to show that,
$$\sum_{S: |\hat{f}(S,\mu)|\leq \beta} \hat{f}^2(S,\mu) \leq 3(1-c)^ds+2^{d+2}\beta.$$
This can be done by breaking it into two cases,
$$\sum_{S: |\hat{f}(S,\mu)|\leq \beta} \hat{f}^2(S,\mu) = \sum_{S: |\hat{f}(S,\mu)|\leq \beta\wedge |\hat{g}(S,\mu)|\geq 2\beta} \hat{f}^2(S,\mu)+\sum_{S: |\hat{f}(S,\mu)|\leq \beta\wedge |\hat{g}(S,\mu)|\leq 2\beta} \hat{f}^2(S,\mu).$$
Each $S$ occurring in the first term above contributes at least $\beta^2$ to $\sum_S (\hat{f}(S,\mu)-\hat{g}^2(S,\mu)\leq (1-c)^d s$, hence there can be at most $(1-c)^ds/\beta^2$ terms in the first term above, and
$$\sum_{S: |\hat{f}(S,\mu)|\leq \beta\wedge |hat{g}(S,\mu)|\geq 2\beta} \hat{f}^2(S,\mu) \leq \beta^2 \frac{(1-c)^d s}{\beta^2}=(1-c)^ds.$$
Using the fact that $(a+b)^2 \leq 2(a^2+b^2)$, for any reals $a,b$, we have,
\begin{align*}
\sum_{S: |\hat{f}(S,\mu)|\leq \beta\wedge |\hat{g}(S,\mu)|\leq 2\beta} \hat{f}^2(S,\mu) &\leq
\sum_{S: |\hat{f}(S,\mu)|\leq \beta\wedge |\hat{g}(S,\mu)|\leq 2\beta} 2\left((\hat{f}(S,\mu)-\hat{g}(S,\mu))^2+\hat{g}^2(S,\mu)\right)
\end{align*}
Now we know that $\sum_{S}(\hat{f}(S,\mu)-\hat{g}(S,\mu))^2 \leq (1-c)^ds$, so this gives an upper bound of $2(1-c)^ds$ on the sum of the first terms in the above.  It suffices to show that,
$$\sum_{S: |\hat{g}(S,\mu)|\leq 2\beta} \hat{g}^2(S,\mu) \leq 2^{d+1}\beta.$$
To see this, note that $g$ has at most $4^d$ nonzero terms, as a depth-$d$ decision tree.  And since any vector $v \in \reals^{4^d}$ with $\|v\|\leq 1$ has $\|v\|_1 \leq 2^d$, we have that $\sum_S |\hat{g}(S,\mu)| \leq 2^d$.  Finally,
$$\sum_{S: |\hat{g}(S,\mu)|\leq 2\beta} \hat{g}^2(S,\mu) \leq \sum_S |\hat{g}(S,\mu)|2\beta \leq 2^{d+1}\beta.\qedhere$$
\end{proof}

\section{Conclusions}\label{sec:conc}

In conclusion, we have shown in a precise sense, that all decision trees are learnable from most product distributions.  The main tool we have is a type of generalization of KM that uses random examples drawn from a (perturbed) product distribution, and works only for terms of degree $O(\log n)$.  Learning decision trees is a clear demonstration of the power of a new model.  However, the questions raised by such a tool are perhaps even more interesting.  First, can one learn DNFs from most product distributions?  Second, can one agnostically learn in these settings, for example can one agnostically learn decision trees in this setting?  A third and very interesting direction would be to go beyond product distributions to arbitrary perturbed distributions.  To be precise, let $\cD$ be an arbitrary distribution on $\11^n$.  Let $a,b \in_\cU [0,c]^n$ be two uniformly random perturbation vectors.  Consider the distribution in which $x$ is first chosen from $\cD$ and then each bit $x_i$ is altered as follows: if $x_i=1$ then $x_i$ is flipped with probability $a_i$, if $x_i=-1$ then $x_i$ is flipped with probability $b_i$.  This gives a new type of perturbed distribution on inputs which is not in general a product distribution.  Hence, our current techniques will not work but it is possible that others will.

Finally, we mention that the Goldreich-Levin algorithm \cite{GoldreichLevin:89}, similar to KM, has a number of applications in computational complexity and other areas.  It would be interesting to see if these applications could also be studied from random examples, instead of black-box access, in a smoothed analysis setting.

\noindent
{\bf Acknowledgments}. We are very grateful to Ran Raz, Ryan O'Donnell, and Prasad Tetali for illuminating discussions.

\bibliographystyle{siam}
\bibliography{ag}

\end{document}